\documentclass[11pt]{article}
\usepackage[margin=1in]{geometry}
\usepackage{natbib}

\usepackage[utf8]{inputenc} 
\usepackage[T1]{fontenc}    
\usepackage{hyperref}       
\usepackage{url}            
\usepackage{booktabs}       
\usepackage{amsfonts}       
\usepackage{nicefrac}       
\usepackage{microtype}      
\usepackage{xcolor}         

\usepackage{amsthm}
\usepackage{thmtools}
\usepackage{thm-restate}

\usepackage{amsmath}
\usepackage{amsthm}
\usepackage{amssymb}
\usepackage{mathtools}
\usepackage{algorithm}
\usepackage{algpseudocode}
\usepackage{xspace}
\usepackage{cleveref}

\newtheorem{theorem}{Theorem}
\newtheorem{lemma}{Lemma}

\newtheorem{corollary}{Corollary}
\newtheorem{example}{Example}
\newtheorem{claim}{Claim}

\title{Computing Thiele Rules on \\Interval Elections and their Generalizations}

\author{%
\begin{tabular}{ccc}
Dimitris Avramidis & Alexandra Lassota & Ulrike Schmidt-Kraepelin \\
\texttt{d.n.avramidis@tue.nl} &
\texttt{a.a.lassota@tue.nl} &
\texttt{u.schmidt.kraepelin@tue.nl}
\end{tabular}
\\[1.5ex]
Department of Mathematics and Computer Science\\
Eindhoven University of Technology\\
Eindhoven, Netherlands
\\[2ex]
Adrian Vetta\\
School of Computer Science, and Department of Mathematics and Statistics\\
McGill University\\
Montréal, Canada\\
\texttt{adrian.vetta@mcgill.ca}
}

\date{}

\newcommand{\lc}{\ensuremath{\mathcal{A}^{\text{LC}}}}
\newcommand{\ci}{\ensuremath{\mathcal{A}^{\text{CI}}}}
\newcommand{\vi}{\ensuremath{\mathcal{A}^{\text{VI}}}}
\newcommand{\oned}{\ensuremath{\mathcal{A}^{\text{VCI}}}}
\newcommand{\tree}{\ensuremath{\mathcal{A}^{\text{TR}}}}
\newcommand{\ilpa}{\ensuremath{\text{ILP}_{(A,k,w)}}\xspace}
\newcommand{\lpa}{\ensuremath{\text{LP}_{(A,k,w)}}\xspace}
\newcommand{\lpap}{\ensuremath{\text{LP}_{(A',k',w')}}\xspace}

\definecolor{sagegreen}{RGB}{112, 130, 96}

\begin{document}

\maketitle

\begin{abstract}
Approval-based committee voting has received significant attention in the social choice community. Among the studied rules, \emph{Thiele rules}, and especially \emph{Proportional Approval Voting} (PAV), stand out for desirable properties such as proportional representation, Pareto optimality, and support monotonicity. Their main drawback is that computing a Thiele outcome is NP-hard in general.
A glimpse of hope comes from the fact that Thiele rules are better behaved under structured preferences. On the \emph{candidate interval} (CI) domain, they are computable in polynomial time via a linear program (LP) that has a totally unimodular constraint matrix. Surprisingly, this approach fails for the related \emph{voter interval} (VI) domain, and the complexity of the problem has repeatedly been posed as an open question. Our main result resolves this question: although the relevant matrix is not totally unimodular, the ``standard'' LP still admits at least one optimal integral solution, and we provide a fast algorithm for finding it. 

Our technique naturally extends to the \emph{voter-candidate interval} (VCI) domain, also known as the {\em 1-dimensional voter-candidate range} (1D-VCR) domain, and to the \emph{linearly consistent}~(LC) domain, both of which generalize the candidate and voter interval domains. Although both the VCI and LC domains have been studied in social choice, their relationship was unknown. We show, through connections to graph theory, that LC strictly contains VCI. We also provide an alternative definition of LC that is closer in spirit to VCI and has a natural interpretation in approval elections; this equivalence may be of independent interest. Finally, we study an alternative tree-based generalization of VCI and show that Thiele rules become NP-hard to compute on this domain.
\end{abstract}

\section{Introduction}\label{sec:intro}

Approval-based committee voting is a fundamental variant of multiwinner voting 
that has received significant attention in social choice theory in recent years 
\citep{lackner23abc_book}. In this setting, voters submit approval ballots over 
a set of candidates, and a voting rule selects a fixed-size subset of candidates, 
called a committee. Voters derive utility from each approved candidate included in 
the selected committee. Approval-based committee elections have real-world 
applications in proof-of-stake blockchain protocols 
\citep{boehmer2024approval,cevallos2021verifiably} and form an important special 
case of participatory budgeting \citep{aziz2020participatory}.

The social choice literature has developed a host of voting rules for 
approval-based committee elections, while also revisiting classical ideas from the 
origins of proportional representation. Among this broad range of rules, one class 
stands out: \emph{Thiele rules}, introduced by \citet{Thiele95}. Thiele rules are 
natural because they select committees that maximize a welfare objective, and, 
remarkably, they were first proposed in the nineteenth century. Formally, they 
maximize the sum of voter scores, where each voter's score depends only on the 
number of approved candidates contained in the committee. The class includes the 
well-known rules of \emph{approval voting} (AV), which maximizes the sum of 
total approvals; the \emph{Chamberlin--Courant} rule (CC), which maximizes the number of voters with at least one approved candidate; and \emph{proportional approval voting} (PAV), which seeks to 
balance these two objectives.

A central concern in approval-based committee elections is a fairness ideal known 
as \emph{proportional representation}, roughly the principle that every 
$\alpha$-fraction of the voters should be represented by an $\alpha$-fraction of 
the committee. The PAV rule has been repeatedly identified as performing 
particularly well in terms of proportionality: it satisfies \emph{extended 
justified representation} \citep{aziz2017justified}, achieves the best-possible 
\emph{proportionality degree} \citep{aziz2018complexity}, and provides a 
$2$-approximation to core stability \citep{peters2020proportionality}\footnote{There exist different notions of approximate core stability. The notion referenced here is 
defined with respect to a relaxation in the utility of the agents, not the size of 
the coalition.}, which is the best possible among rules satisfying the so-called 
\emph{Pigou--Dalton} principle. Furthermore, in the related setting where candidates 
may be selected multiple times, PAV in fact satisfies \emph{core stability} exactly 
\citep{brill2024approval} and is the only rule known to do so. 

Unfortunately, despite these desirable properties, Thiele rules have one 
significant practical drawback: essentially all of them (besides AV) are NP-hard 
to compute \citep{SkowronFL16_aij_set_of_items}. To address this, greedy variants 
\citep{aziz2017justified} and approximation algorithms \citep{dudycz2021tight} have 
been studied, though these relaxations do not preserve the 
strong normative properties discussed above. In this paper, we follow a third approach: we study 
Thiele rules under restricted domains. In particular, \citet{DBLP:conf/aaai/Peters18} 
showed that Thiele rules can be computed efficiently for so-called \emph{candidate 
interval}~(CI) elections, where candidates can be arranged on a line and each voter 
approves a consecutive subset of them. The proof proceeds by showing that an ILP 
formulation of the problem has a totally unimodular constraint matrix; consequently, an 
optimal solution can be found by solving the LP relaxation. Interestingly, this 
technique does not extend directly to the natural \emph{voter interval}~(VI) domain, 
where voters are arranged on a line and each candidate is approved by a consecutive 
subset of them. The complexity of computing Thiele rules in the voter interval 
domain was consequently repeatedly stated as an open question 
\citep{ElkindL15_ijcai15, DBLP:conf/aaai/Peters18,lackner23abc_book, lassota2026algorithms}. We resolve this question for voter interval elections and beyond. 

\paragraph{Our Contribution.}
Our main result, presented in \Cref{sec:CandInt}, shows that all Thiele rules can be computed in polynomial time for voter-candidate interval elections (VCI elections, also called \emph{1D-VCR} elections \citep{ELKIND2023114039}), a class of elections generalizing both VI and CI elections. To prove our main result, we revisit the ILP formulation of~\cite{DBLP:conf/aaai/Peters18} and show that the corresponding LP relaxation always admits at least one optimal solution that is integral. Our proof is constructive: in particular, we provide a polynomial-time algorithm that finds such an integral solution. When the weight function of the Thiele rule is strictly decreasing and positive, as is the case for PAV, we strengthen this result further by proving that every optimal extreme point of the LP is integral. Thus, in such instances, finding a basic optimal solution of the LP directly suffices to compute an outcome of the Thiele rule.

In \Cref{sec:LC}, we extend our main result beyond VCI elections. The key property underlying our main result is that, after removing dominated candidates from a VCI election, the resulting election is candidate-interval (CI). Consequently, our result applies to any class of elections with this property. We show that this includes the class of linearly consistent (LC) profiles, introduced by \cite{10.1007/978-3-031-22832-2_18}. Perhaps surprisingly, the relationship between the VCI and LC domains was previously open; we show that LC profiles in fact generalize VCI elections. To prove this, we leverage connections to interval graphs, which also yield an alternative interpretation of LC elections as \emph{interval containment} graphs. We find this formulation arguably more natural for approval elections than the original definition of LC, and therefore believe that the equivalence is of independent interest.

Finally, in \Cref{sec:tree}, we study the complexity of Thiele rules for an orthogonal generalization of the VCI domain: approval elections defined via a tree representation. In particular, we identify a boundary of tractability for Thiele rules by showing that the resulting decision problem is NP-hard.

\paragraph{Further Related Work.}
Thiele rules on restricted domains were first studied by \cite{ElkindL15_ijcai15}. They showed that PAV is fixed-parameter tractable for the VI and CI domains, both parameterized by the maximum number of approvals per voter or the maximum number of supporters of any candidate. Arguably simpler algorithms have been recently introduced by \cite{lassota2026algorithms} for the same parameterizations. In their paper, \cite{ElkindL15_ijcai15} also conjectured that the problem is NP-hard on both the VI and CI domains. This conjecture was partially disproved by \cite{DBLP:conf/aaai/Peters18}, who gave a linear programming formulation whose constraint matrix is totally unimodular on the CI domain. Ergo, the LP admits integral optimal solutions, which correspond to winning committees under Thiele rules. Unfortunately, \cite{DBLP:conf/aaai/Peters18} also showed that the analogous total-unimodularity argument fails for the VI domain. More recently, \cite{lassota2026algorithms} showed that Thiele rules can be computed efficiently when each candidate is approved by at most two voters, and gave an FPT algorithm parameterized by the total score of a winning committee. Other FPT results were achieved by \cite{SornatWX22_ijcai,BredereckF0KN20_aaai,FaliszewskiSST18_scw}. 

The tractability of Thiele rules on the voter interval (VI) domain was concurrently and independently resolved by \cite{manurangsi2026polynomial}. Their result, obtained via human--AI collaboration, uses binary search over the multiplier of a Lagrangian relaxation of the LP relaxation of \cite{DBLP:conf/aaai/Peters18}. While we also use the LP relaxation by \cite{DBLP:conf/aaai/Peters18} as our starting point, we do not use the Lagrangian relaxation and instead directly establish the existence of integral optimal solutions. Consequently, we find our algorithm and our arguments more straightforward and simpler. Moreover, our technique applies directly to all restricted domains satisfying the aforementioned key property. It therefore immediately extends to two superclasses of voter interval domains: the VCI and the LC domains.

\section{Preliminaries}\label{sec:prelim}

We are given $n \in \mathbb{N}$ voters and $m \in \mathbb{N}$ candidates and denote the corresponding sets by $[n]$ and $[m]$, respectively, where $[k] = \{1, \dots, k\}$ for every integer $k \in \mathbb{N}$. 
For each voter $i \in [n]$, we define its set of approved candidates by $C_i \subseteq [m]$. For each candidate $j \in [m]$, we define its set of supporters as $N_j = \{i \in [n] \mid j \in C_i\}$. For convenience, we also capture this information via an \emph{approval matrix} defined as the $n \times m$ matrix $A$ with binary entries $A[i][j] = 1$ if $j \in C_i$ and $0$ otherwise. In addition, we are given an integer $k \in \mathbb{N}$, which we call the \emph{committee size}, and any subset $W \subseteq C$ with $|W|=k$ is called a committee. An approval-based committee voting election now consists of the approval matrix and committee size, i.e., $(A,k)$. 
We say that candidate $j$ \emph{dominates} candidate $j'$ if $N_{j'} \subset N_j$. We call a candidate $j$ \emph{undominated} if no other candidate $j'$ dominates it and we call an approval matrix $A$ \emph{domination-free} if all candidates are undominated. 

\paragraph{Restricted Domains.}
In this paper, we study approval-based committee elections under structural assumptions about the approval matrix. The following three subclasses will be of relevance: 
\begin{itemize}
\item The class of \emph{candidate-interval} matrices $\ci$ containing all $0$-$1$-matrices for which we can reorder the columns such that all $1$-entries in each row form an interval. The name comes from the fact that these matrices correspond to elections with the candidate-interval property, as defined by \cite{ElkindL15_ijcai15}. 
\item The class of \emph{voter-interval} matrices $\vi$ containing all matrices for which we can reorder the rows such that all $1$-entries in each column form an interval. The name comes from the fact that these matrices correspond to elections with the voter-interval property, as defined by \cite{ElkindL15_ijcai15}. 
\item The class of \emph{voter-candidate-interval (VCI)} matrices $\oned$ containing all matrices for which there exists a so-called \emph{1D-voter-candidate range} explanation. That is, we can assign each voter and each candidate an interval on the real line $\mathbb{R}$, such that, for every voter $i$ and candidate $j$ it holds that $A[i][j] = 1$  if and only if the two corresponding intervals intersect. Such elections have been studied, for example, by \cite{GodziszewskiB0F21_vci_aaai21} and \cite{dong25interlacing}.
\end{itemize} 
It has been shown that $\ci \subseteq \oned$ and $\vi \subseteq \oned$, but neither $\ci \not\subseteq \vi$ nor $\vi \not\subseteq \ci$. 

There are two broader subclasses of particular relevance for this work, namely the \emph{linearly consistent} domain and the \emph{tree representation} domain. We defer the definitions of these domains to~\Cref{sec:LC} and \Cref{sec:tree}, respectively, when they become relevant.

\paragraph{Thiele Rules.}

Thiele rules refer to a class of voting rules using an optimization approach, introduced by \cite{Thiele95}. Given a typically non-increasing weight vector $w=(w_1, \dots, w_k)$, a committee $W$ is winning under the $w$-Thiele rule, if it maximizes the following score 
\begin{equation}
\sum_{i\in[n]} \sum_{\ell=1 }^{|W \cap C_i|} w_{\ell}. \label{pav-score}
\end{equation}
Three classical weight vectors considered in the literature are 
(i) $(1,1,\dots,1)$ leading to the \emph{approval voting (AV)} rule,
(ii) $(1,0,\dots,0)$ leading to the \emph{Chamberlin-Courant (CC)} rule, and (iii)~$(1,\frac{1}{2},\frac{1}{3}, \dots, \frac{1}{k})$, leading to the \emph{proportional approval voting (PAV)} rule. 
The NP-hardness of CC has been shown by \cite{procaccia2008complexity}, and the NP-hardness of PAV has been shown by \cite{SkowronFL16_aij_set_of_items} and \cite{AzizGGMMW15} (for different decision problems). In the following, we generalize this definition by considering personalized weight vectors for each voter $i$, i.e., $w^i=(w^{i}_1, w^{i}_2, \dots, w^{i}_k)$. We study the {\em committee winner determination problem}: given an approval election $(A,k)$ and weight functions $w=(w^i)_{i \in [n]}$, find a committee $W$ that maximizes the score~(\ref{pav-score}).
We refer to $(A,k,w)$ as an instance of the problem.  %

\section{A Polynomial Time Algorithm for Generalized Thiele Voting Rules in the Voter-Candidate Interval Domain}\label{sec:CandInt}

In this section, we present our main result, a polynomial time algorithm for the committee winner determination problem in the voter-candidate interval (VCI) domain.
Specifically we construct an algorithm that, for any set of non-negative and non-increasing personalized weight vectors $w = (w^i)_{i \in [n]}$ with $w_1^i \geq w_2^i \geq \dots \geq w_k^i\geq 0$, finds a committee winning under $w$-Thiele.
Our algorithm builds upon the following LP-formulation, which has already been used by \cite{DBLP:conf/aaai/Peters18} to show that $w$-Thiele rules can be computed in polynomial time for the candidate-interval (CI) domain. 

\paragraph{The Integer Linear Program.}
Let $x_j \in \{0,1\}$ indicate whether candidate $j$ is selected. Further, let $y^i_\ell \in \{0,1\}$ denote whether voter $i$ approves at least $\ell$ selected candidates. For a given instance~$(A,k)$ and weight functions~$w$, we define $\ilpa$:
\begin{align}
\text{maximize} \quad & \sum_{i=1}^{n} \sum_{\ell=1}^{k} w^{i}_\ell \cdot y^i_\ell \label{const1}\\
\text{subject to} \quad 
& \sum_{j = 1}^m x_j = k \label{const2}\\
& \sum_{j \in C_i} x_j = \sum_{\ell = 1}^k y^i_\ell \quad \forall i \in [n] \label{const3}\\
& x_j \in \{0,1\} \quad \forall j \in [m]\\
& y^i_\ell \in \{0,1\} \quad \forall i \in [n], \ell \in [k]
\end{align}

Constraint~\ref{const2} ensures that exactly $k$ candidates are selected. 
Since the personalized weight vectors are non-increasing,
Constraint~\ref{const3} enforces consistency between the number of selected candidates approved by voter $i$ and the auxiliary variables $y^i_\ell$. This ensures that $y^i_\ell$ correctly captures whether voter $i$ has at least $\ell$ approved candidates in the committee. The objective function~\ref{const1} simply sums the contribution of each voter according to the number of approved candidates they receive. 

We refer to the LP-relaxation of \ilpa as \lpa. Clearly, if an optimal solution to \lpa  is integral, it is also an optimal solution to \ilpa. The following result is due to \cite{DBLP:conf/aaai/Peters18}: 

\begin{theorem}[\cite{DBLP:conf/aaai/Peters18}] \label{lem:dominik}
    For a given instance $(A,k)$ and weight functions $w$ for which $A \in \ci$, it holds that the constraint matrix of \lpa is totally unimodular. Hence, every optimal extreme point of \lpa is integral and such a point can be found in polynomial time. 
\end{theorem}
We remark that \cite{DBLP:conf/aaai/Peters18} studied the case of identical weight vectors, but the result easily extends to
personalized weight vectors.
Given the above result, it is natural to ask whether a similar result holds for the classes $\vi$ and $\oned$. Unfortunately this is not the case, as already observed by \cite{DBLP:conf/aaai/Peters18}. For completeness, we provide an example below which we will use again later: 

\begin{example}\label{ex1}
    Consider the following approval matrix: 
    \[A = \begin{pmatrix}
    1 & 1 & 0 & 0\\
    1 & 0 & 1 & 0\\
    1 & 0 & 0 & 1
    \end{pmatrix}
    \]
    Observe that $A \in \vi$, and $A \not\in\ci$. If we add an all-ones row to the matrix, the extended matrix has determinant $-2$. Moreover this extended matrix appears as a submatrix in the constraint matrix of \lpa. Hence, the constraint matrix of \lpa is not totally unimodular. 
\end{example}

However, perhaps surprisingly, we can show that \lpa has at least one optimal solution that is also integral. Furthermore, such a solution can be found in polynomial time by Algorithm~\ref{alg:main}, stated below. This yields our main result: 
\begin{theorem}\label{thm:main}
    For any instance $(A,k)$ with $A \in \oned$ and any non-increasing weight functions $w$, \lpa has an optimal solution that is integral, which Algorithm \ref{alg:main} finds in polynomial time. 
\end{theorem}

Algorithm~\ref{alg:CandInt} proceeds in two key steps. First, we solve \lpa to obtain a (possibly fractional) optimal solution. Then, candidates that receive integral values in this solution are fixed: those with value~$1$ are included in the committee, while those with value~$0$ are excluded.
The remaining candidates, i.e., those assigned strictly fractional values, induce a subinstance. We will show that, after the modification performed in \Cref{alg:main}, this residual instance corresponds to a domination-free approval matrix. Using an observation by \cite{dong25interlacing}, such approval matrices are included in the candidate-interval subdomain: 

\begin{lemma}[\cite{dong25interlacing}]\label{lem:chris}
If $A \in \oned$ is a domination-free approval matrix then $A \in \ci$. 
\end{lemma}

Thus, such instances can be solved using \Cref{lem:dominik}. Ultimately, our proof shows that following this two-step procedure induces an optimal and integral solution of \lpa, and thus an optimal solution to \ilpa. 

\begin{algorithm}
\caption{Thiele Rules on the Voter-Candidate Interval Domain}
\label{alg:CandInt}
\begin{algorithmic}[1]
\Require Voters $[n]$, candidates $[m]$, approval matrix $A$, weight vectors $(w^i)_{i \in [n]}$, committee size $k$
\Ensure A committee $W \subseteq [m]$ of size $k$
\State Solve \lpa to obtain an optimal, potentially fractional solution $(x, y)$ \label{line:LP}
\While{there exists $j,j' \in [m]$ such that $x_j < 1$, $x_{j'}>0$  and $j$ dominates $j'$}
\State Shift weight from $x_{j'}$ to $x_j$ until either $x_j = 1$ or $x_{j'} = 0$  \label{line:shift}
\EndWhile
\State Set $W_0 = \{j \in [m] \mid x_j = 0\}$, $W_1 = \{j \in [m] \mid x_j = 1\}$, $C' = [m] \setminus \{W_0 \cup W_1$\}, and $k' = k - |W_1|$
\State Set $\ell_i = |C_i \cap W_1|$ and define $w'^i_\ell = w^i_{\ell + \ell_i}$ for all $i \in [n], \ell \in [k']$
\State Set $A'$ to be the submatrix of $A$ induced by the candidates (i.e. columns) $C'$
\State Solve the instance $(A',k',w')$ via \Cref{lem:dominik}, obtain $(x',y')$ and set $W' = \{j \in C' \mid x'_j = 1\}$ \label{line:residual}
\State \Return $W_1 \cup W'$
\end{algorithmic}
\label{alg:main}
\end{algorithm}

Before proving our main result, two comments are in order here. First, the while loop of Algorithm~\ref{alg:CandInt}
is not fully specified. To ensure a polynomial running time, as we explain later, given multiple feasible choices we must judiciously select the pair $\{j,j'\}$. 
Second, the restriction to matrices in the voter-candidate interval domain $\oned$ is not necessary.
As we show in Section~\ref{sec:LC}, the algorithm succeeds for the linearly consistent (LC) domain
and, indeed, for any domain that satisfies an analogue of
Lemma~\ref{lem:chris}.

\begin{proof}[Proof of \Cref{thm:main}]
Consider an optimal fractional solution $(x,y)$ to \lpa. We start with two structural observations about $(x,y)$ in Claims~\ref{claim-y} and \ref{claim:shiftingX}. These observations also establish stronger statements that hold under stronger assumptions on $w$. While these stronger versions are not strictly necessary for the present proof, they arise naturally and will be used later in the proof of \Cref{thm:extreme}.
We first show that we can assume w.l.o.g. that $y$ has a simple form: 
\begin{claim}
\label{claim-y}
    For any feasible solution $(x,y)$, the following $y^*$ leads to an objective value of $(x,y^*)$ in \lpa that is at least as high as $(x,y)$: 
    $$y^{*i}_{\ell} = \begin{cases} 1 & \text{for } \ell \leq \lfloor \sum_{j \in C_i} x_j\rfloor \\ \sum_{j \in C_i} x_j - \lfloor \sum_{j \in C_i} x_j\rfloor & \text{for } \ell = \lfloor \sum_{j \in C_i} x_j\rfloor  +1\\ 0 & \text{for } \ell > \lfloor \sum_{j \in C_i} x_j\rfloor + 1 \end{cases}$$
    Moreover, when all weight functions $(w^i)_{i \in [n]}$ are strictly decreasing, then either $y=y^*$ or the objective value of $(x,y^*)$ in \lpa is strictly higher than that of $(x,y)$.
\end{claim}

\begin{proof}[Proof of Claim]
  \renewcommand{\qedsymbol}{$\textcolor{teal}{\blacksquare}$}
  Take any feasible solution $(x,y)$ and assume $y \neq y^*$. Then, for each voter $i$, we can independently do the following shifting operation without changing the sum of the $y$-variables for this voter, i.e., $\sum_{\ell=1}^k y_{\ell}^i$. As $y\neq y^*$, there exists $\ell, \ell' \in [k]$, where $\ell<\ell'$, such that $y_{\ell}^i < 1$ and $y_{\ell'}^i >0$. Shifting weight from the latter variable to the former variable does not decrease the objective value, as the personalized weight function $w^i$ is non-increasing. Moreover, if the weight function is strictly decreasing, this shifting operation strictly increases the objective value. The claim follows. 
\end{proof}

We also refer to the objective value of $x$, which is formally the objective value of $(x,y^*)$ in \lpa, where $y^*$ is defined as in the claim above. Note that the objective value of $x$ only depends on the values $r_i(x) = \sum_{j \in C_i} x_j$ for all $i$, which we refer to as the \emph{representation value} of $i \in [n]$. 

\begin{claim}\label{claim:shiftingX}
    The shift operation in Line \ref{line:shift} of \Cref{alg:main} does not decrease the objective value of $x$. Moreover, when all weight functions $(w^i)_{i \in [n]}$ are strictly positive, no pair of candidates $j,j'$ can satisfy the condition in the while loop. 
\end{claim}

\begin{proof}[Proof of Claim]
    \renewcommand{\qedsymbol}{$\textcolor{teal}{\blacksquare}$}
    Note that the shifting operation in Line \ref{line:shift} strictly increases the representation value~$r_i$ for some $i \in [n]$ while it does not decrease the representation value of any other voter. Since all weight functions are non-negative, this operation does not decrease the objective value of $x$. Moreover, if all $(w^i)_{i \in [n]}$ are strictly positive, the shifting operation would strictly increase the objective value of $x$, a contradiction to $(x,y)$ being an optimal solution to \lpa. Hence, the condition in the while loop cannot be satisfied for any pair $j,j' \in [m]$. 
\end{proof}

Let $x$ be the vector after the shift operation, which by \Cref{claim:shiftingX} still has optimal objective value. Let $y^*$ be the corresponding vector according to \Cref{claim-y}, and let $(A',k',w')$ be the residual instance constructed in \Cref{alg:main}. First, note that the property $A' \in \oned$ is inherited by $A \in \oned$ and by the termination of the shifting operation, we also know that $A'$ is domination-free. Thus, by \Cref{lem:chris}, it holds that $A' \in \ci$. Hence, by \Cref{lem:dominik} there exists $(x',y')$ that is optimal for $\text{LP}_{(A',w',k')}$ and integral. 

Now, let $f$ be the function mapping a feasible solution of \lpa to its objective value and similarly, let $f'$ be the function mapping a feasible solution of \lpap to its objective value. We slightly abuse notation, by also allowing $f$ to take sets of candidates as its input, which corresponds to applying $f$ to the characteristic vector of that set. 

\begin{claim}
It holds that $f(x) = f(W_1) + f'(x') = f(W_1) + f'(W') = f(W_1 \cup W')$. 
\end{claim}

\begin{proof}[Proof of Claim]
\renewcommand{\qedsymbol}{$\textcolor{teal}{\blacksquare}$}
    We start by showing $f(x) \leq f(W_1) + f'(x')$. 
    We first define $(\hat{x},\hat{y})$ as the fractional part of the solution $(x,y^*)$. 
    Formally, this means $\hat{x}_j = x_j$ if $j \in C'$ and $x_j = 0$ else. It also means $\hat{y}_{\ell}^i = y_{\ell_i + \ell}^i$ for all $\ell \in [k']$ and $i \in [n]$, where $\ell_i=|C_i\cap W_1|$. Note that $(\hat{x},\hat{y})$ is a feasible solution for \lpap. Moreover, the objective value of $(x,y^*)$ is 
    $$f(x)  \stackrel{(1)}{=} \sum_{i \in [n]} \sum_{\ell =1 }^{\ell_i} w_{\ell}^i + \sum_{i \in [n]} \sum_{\ell 
    = \ell_i + 1}^k w_{\ell}^i \cdot y_{\ell}^{*i} \stackrel{(2)}{=} f(W_1) + \sum_{i \in [n]} \sum_{\ell 
    = 1}^{k'} {w'}_{\ell}^{i}\cdot  \hat{y}_{\ell}^i
    \stackrel{(3)}{=} f(W_1) + f'(\hat{x}) \stackrel{(4)}{\leq} f(W_1) + f'(x'),$$
    where (1) follows from the definition of $\ell_i$; (2) follows from the definitions of $w'$ and $\hat{y}$, and the fact that $y_{\ell}^{*i} = 0$ for all $\ell > \ell_i + k'$, since every voter receives exactly representation $\ell_i$ from $W_1$ and at most representation $k'$ from candidates outside of $W_1$; (3) follows since $(\hat{x},\hat{y})$ is a feasible solution of \lpap; and (4) follows since $(x',y')$ is an optimal solution to \lpap. 

\smallskip 

We continue by arguing that $f(W_1) + f'(x') = f(W_1) + f'(W')$ since $x'$ is the characteristic vector of $W'$ as $x'$ is integral. 

\smallskip

Lastly, we show that $f(W_1) + f'(W') = f(W_1 \cup W')$. To this end, define $\ell'_i = \sum_{j \in C_i} x'_j$, which is the representation that $i$ receives from $W'$. Then, $$f(W_1) + f'(W') = \sum_{i \in [n]}\sum_{\ell=1}^{\ell_i}w_{\ell}^i + \sum_{i \in [n]}\sum_{\ell=1}^{\ell'_i}{w'}_{\ell}^i = \sum_{i \in [n]}\sum_{\ell=1}^{\ell_i}w_{\ell}^i + \sum_{i \in [n]}\sum_{\ell=\ell_i + 1}^{\ell_i + \ell'_i}{w}_{\ell}^i = f(W_1 \cup W'),$$ which holds directly by the definition of the weight functions $({w'}^i)_{i \in [n]}$. 

\smallskip

Summarizing, we get that $f(x) \leq f(W_1 \cup W')$. However, since $W_1 \cup W'$ corresponds to a feasible solution to \lpa and $(x,y^*)$ is an optimal solution to \lpa, this implies $f(x) = f(W_1 \cup W')$. 
\end{proof}

Hence, $W_1 \cup W'$ corresponds to an optimal and integral solution to \lpa, which proves the first part of \Cref{thm:main}. It remains to argue about the running time of \Cref{alg:main}.

\smallskip 

To prove the running time is polynomial there are three critical points to analyze. First, in Line~\ref{line:LP}, the linear program \lpa has $O(nk + m)$ variables and constraints, and all coefficients in the objective function and constraints are polynomially bounded. 
It is well known that the corresponding linear relaxation can be solved in polynomial time with respect to their encoding length, for instance via the ellipsoid method of~\cite{Kha79} (see also~\cite{DBLP:books/daglib/0090562},~\cite{GLS93}). 
Similarly, after fixing the integral variables, the residual instance in Line~\ref{line:residual} is domination-free, as argued above. By \Cref{lem:dominik}, such instances can be solved in polynomial time.
Third, we can polynomially bound the number of implementations of the while loop if we carefully select how the weight is shifted. We say $\{j, j'\}$ is a feasible {\em recipient-donor pair} if
$x_j<1$, $x_{j'}>0$ and $j$ (recipient) dominates $j'$ (donor).
Among all feasible pairs, we switch weight from a donor $j'$ to a recipient $j$ that maximizes $|N_j\setminus N_{j'}|$. By doing so, we claim the total number of feasible pairs strictly decreases. To see this, observe that after the weight switch either $x_j=1$ or $x_j=0$.
In the former case, $j$ can no longer be a recipient in any feasible recipient-donor pair as now $x_j=1$. Moreover,
we cannot create a new recipient-donor pair $\{j'',j\}$ in which $j$ is the donor because then
$\{j'',j'\}$ was a feasible recipient-donor pair with $|N_{j''}\setminus N_{j'}|> |N_j\setminus N_{j'}|$, a contradiction. In the latter case, 
$j'$ can no longer be a donor
in any feasible recipient-donor pair as now $x_{j'}=0$. Moreover,
we cannot create a new recipient-donor pair $\{j',j''\}$ in which $j'$ is the recipient because then
$\{j,j''\}$ was a feasible recipient-donor pair with $|N_{j}\setminus N_{j''}|> |N_j\setminus N_{j'}|$, a contradiction. Thus the while loop is applied at most a polynomial number of times. 
The remaining steps of the algorithm trivially run in polynomial time, completing the proof.
\end{proof}

Having established that \lpa always admits an optimal solution that is also integral, it is natural to ask whether all optimal extreme points of \lpa share this property. If so, one could replace \Cref{alg:main} with the simpler procedure of finding any optimal extreme point of \lpa, which would automatically yield an integral solution. Unfortunately, the following example shows that this stronger statement does not hold in general.

\begin{example}
    Consider the same approval matrix as in \Cref{ex1}, i.e.,\[A = \begin{pmatrix}
    1 & 1 & 0 & 0\\
    1 & 0 & 1 & 0\\
    1 & 0 & 0 & 1
    \end{pmatrix}, 
    \]
    let $k=2$, and let $(w^i)$ be defined as $w^i_{\ell} = 0$ for all $i \in [n], \ell \in [k]$. Then, consider the point $x_j=1/2$ for all $j \in [4]$, $y_{1}^i = 1, y_2^i = 0$ for all $i \in [3]$. Clearly, this point is optimal and we argue in the following that it is also an extreme point. To this end, note that there are ten active constraints that uniquely determine these ten variables in the solution: The six active $y$-constraints and the following four constraints: \begin{align*}
        x_1 + x_2 &= 1 \\ 
        x_1 + x_3 &= 1 \\ 
        x_1 + x_4 &= 1 \\ 
        x_1 + x_2 + x_3 + x_4 &= 2. 
    \end{align*}
    Note that subtracting the former three equalities from the last equality gives $-2x_1 = -1$. Thus $x_1 = 1/2$ inducing $x_j = 1/2$, for all $j \in [4]$. It follows that this is indeed an extreme point.
\end{example}

Having a weight function that is all-zero might be unsatisfactory.
Note, however, that the example extends easily to the Chamberlin--Courant weight functions
(i.e., $w_1 = 1$ and $w_\ell = 0$ for all $\ell > 1$) by adding a single candidate approved by all three voters and giving it a value of $1$ in its $x$-variable, and setting $k=3$.
In contrast, we show below that for strictly positive and decreasing weight functions, every optimal extreme point of \lpa must be integral in the $x$-variables.

\begin{theorem}
    For any instance $(A,k,w)$ with $A \in \oned$ and strictly decreasing and positive weight functions $(w^i)_{i \in [n]}$, every optimal extreme point of \lpa is integral. \label{thm:extreme}
\end{theorem}
\begin{proof}
    Let $(x,y)$ be an optimal extreme point of \lpa that is not integral.
By \Cref{claim-y}, $y$ can be transformed into $y^*$, where, for each voter $i \in [n]$, at most one $y^*$-variable is non-integral.
This immediately implies that if $x$ is integral, then $y^*$ must be integral as well,
since the $y^*$-variables of each voter must sum to an integer which is not possible if exactly one of them is non-integral.

Hence, in the remainder of the proof we assume that $x$ is non-integral. Since the weight functions are positive, we know that there does not exist a pair of candidates $j,j' \in [m]$ such that $x_j < 1$ and $x_{j'}>0$ and $j$ dominates $j'$. Now, consider the residual instance $(A',k',w')$ exactly as defined in~\Cref{alg:main} and define $(\hat{x},\hat{y})$ as the fractional part of $(x,y)$, where $y$ satisfies the properties of $y^*$. Formally, this means $\hat{x}_j = x_j$ if $j \in C'$, where $C'$ is defined as in \Cref{alg:main} and $\hat{y}$ as $\hat{y}_{\ell}^i = y_{\ell_i + \ell}^i$ for all $\ell \in [k']$ and $i \in [n]$. Now, note that $(\hat{x},\hat{y})$ is a feasible solution for \lpap and note that $\hat{x}$ is also non-integral. Since $A' \in \ci$, we know by \Cref{lem:dominik} that the constraint matrix of \lpap is totally unimodular and therefore $(\hat{x},\hat{y})$ cannot be an extreme point because every extreme point of \lpap is integral. Thus, there exist two vectors $\Delta \hat{x}$ and $\Delta \hat{y}$, such that $(\hat{x} + \Delta\hat{x}, \hat{y} + \Delta \hat{y})$ and $(\hat{x} - \Delta \hat{x}, \hat{y} - \Delta \hat{y})$ are both feasible solutions for \lpap. 

We are going to use $(\Delta\hat{x},\Delta\hat{y})$ to show that $(x,y)$ is not an extreme point of \lpa. To this end, we first observe that the feasibility of $(\hat{x} + \Delta\hat{x}, \hat{y} + \Delta \hat{y})$ and $(\hat{x}, \hat{y})$ implies that 
\begin{align}
\sum_{j \in C'} \Delta \hat{x}_j \label{extreme1} &= 0 \\
\text{and}\qquad\qquad\nonumber \\ 
    \sum_{\ell =1}^{k'} \Delta \hat{y}_{\ell}^i \label{extreme2} &= \sum_{j \in C_i \cap C'} \Delta \hat{x}_j \quad && \text{for all } i \in [n]. 
\end{align}
Next we define the $m$-dimensional vector $\Delta x$ as $\Delta x_j = \Delta \hat{x}_j$ for $j \in C'$ and $\Delta x_j = 0$ for $j \in [m] \setminus C'$. We also define $\Delta y_{\ell}^i = 0$ for all $i \in [n]$ and $\ell \in [\ell_i]$ and $\Delta y_{\ell}^i = \Delta \hat{y}_{\ell - \ell_i}^i$ for all $i \in [n]$ and $\ell \in \{\ell_i+1, \dots, k\}$. Then we can show that both $(x + \Delta x, y + \Delta y)$ as well as $(x - \Delta x, y - \Delta y)$ are feasible for \lpa. First, Constraint \ref{const2} follows directly from \Cref{extreme1} since $$\sum_{j \in [m]} x_j = \sum_{j \in [m]} (x_j + \Delta x_j) = \sum_{j \in [m]} (x_j - \Delta x_j) = k.$$ Second, Constraint \ref{const3} follows from \Cref{extreme2} since $$\sum_{\ell = 1}^k (y_{\ell}^i + \Delta y_{\ell}^i) = \sum_{\ell = 1}^k y_{\ell}^i + \sum_{\ell = 1}^{k'} \Delta \hat{y}_{\ell}^i \stackrel{(\ref{extreme2})}{=} \sum_{j \in C_i} x_j  + \sum_{j \in C_i \cap C'} \Delta \hat{x}_j = \sum_{j \in C_i} (x_j + \Delta \hat{x}_j).$$ 

To see why all variables are within the interval $[0,1]$, note that $x_j + \Delta x_j = \hat{x}_j + \Delta\hat{x}_j$ for all $j \in C'$ and $x_j + \Delta x_j = x_j$ for all $j \in [m] \setminus C'$. Similarly, for all $i \in [n]$ it holds that $y_{\ell}^i + \Delta y_{\ell}^i = y_{\ell}^i$ for $\ell \leq \ell_i$ and $y_{\ell}^i + \Delta y_{\ell}^i = \hat{y}_{\ell - \ell_i}^i + \Delta \hat{y}_{\ell - \ell_i}^i$ for $\ell > \ell_i$. Thus, the feasibility of the variables is inherited by the feasibility of $(x,y)$ and $(\hat{x} + \Delta \hat{x}, \hat{y} + \Delta \hat{y})$. The feasibility of $(x - \Delta x, y - \Delta y)$ follows analogously. 

Hence, we have shown that $(x,y)$ is not an extreme point of \lpa, yielding the desired contradiction. 
\end{proof}

\Cref{thm:extreme} applies to the important special case of the PAV voting rule, which uses weight functions $w_\ell = \frac{1}{\ell}$ for all $\ell \in [k]$. Consequently, we have shown that on the voter-candidate interval domain, PAV can be solved by finding any optimal extreme point of \lpa.

\section{Linearly Consistent Profiles}\label{sec:LC}

In this section, we show that \Cref{alg:main} extends to the class of linearly consistent (LC) profiles, which were introduced by \citep{10.1007/978-3-031-22832-2_18}. While the relationship between the VCI domain and the LC domain was previously unknown, we show that VCI is a subdomain of LC and highlight an alternative definition of the LC domain.

\cite{10.1007/978-3-031-22832-2_18} introduced linear consistency for the more general class of weak orders. In our setting, they are defined as follows: We say that an approval matrix (or instance) is {\em linearly consistent} if there exists a linear order $\succ$ on the set of voters and candidates such that for any voters $i \succ j$ and candidates $a \succ b$, if $A[i][b]=1$ and $A[j][a]=1$ then also $A[i][a]=1$. We then denote the class of \emph{linearly consistent} matrices by $\mathcal{A}^{\text{LC}}$. We remark that an equivalent definition for bipartite graphs was introduced by \cite{SPINRAD1987279} under the term \emph{bipartite permutation graphs}.

When proving \Cref{thm:main,thm:extreme}, we observed that the specific property we required for the $\oned$ domain was the fact that \Cref{lem:chris} holds. In the following we show that this lemma extends to the LC domain.

\begin{lemma}
If $A \in\mathcal{A}^{\text{LC}}$ is a domination-free approval matrix then $A \in \ci$. 
\end{lemma}
\begin{proof}
Since \(A \in \mathcal{A}^{\text{LC}}\), there exists a linear order \(\succ\) on the voters and candidates satisfying the stated property. We show that, if \(A\) is \emph{domination-free} then the order \(\succ\) also witnesses the candidate-interval property.
Order the voters and candidates according to \(\succ\). We claim that, with respect to this order, the \(1\)-entries in each row form an interval. Suppose not. Then there exists a voter~\(i\) and candidates \(a \succ b \succ c\) such that $A[i][a]=1, A[i][b]=0, A[i][c]=1$.
We claim that candidate \(a\) dominates candidate \(b\). Let $j\neq i$ be any voter with \(A[j][b]=1\).   
    
\smallskip
    \noindent\textbf{Case 1.} \(i \succ j\). Then since \(b \succ c\), \(A[i][c]=1\), and \(A[j][b]=1\), the linearly-consistent property implies that $A[i][b]=1$, a contradiction.
    
\smallskip
    \noindent\textbf{Case 2.} \(j \succ i\). Now, since \(j \succ i\), \(a \succ b\), \(A[j][b]=1\), and \(A[i][a]=1\), the linearly-consistent property implies that $A[j][a]=1$. In particular,
    every voter who approves of candidate \(b\) also approves of candidate \(a\). Therefore candidate \(a\) dominates candidate \(b\), contradicting the assumption that \(A\) is \emph{domination-free}. 
    
Thus, the \(1\)-entries in every row form an interval. Thus, \(A \in \ci\).
\end{proof}

Therefore our proofs of \Cref{thm:main,thm:extreme} immediately apply to the LC domain. In particular, we obtain the following analogous statements. 

\begin{theorem}\label{thm:main2}
    For any instance $(A,k)$ with $A \in \lc$ and any non-increasing weight functions $w$, \lpa has an optimal solution that is integral, which Algorithm \ref{alg:main} finds in polynomial time. Moreover, for any instance $(A,k)$ with $A \in \lc$ and strictly decreasing and positive weight functions $(w^i)_{i \in [n]}$, every optimal extreme point of \lpa is integral.
\end{theorem}

\paragraph{The Relationship between the LC and VCI Domains.}

Both the LC and VCI domains have previously been studied in computational social choice. However, to the best of our knowledge, their relationship has perhaps surprisingly remained open. Interestingly though, closely related concepts have been studied in graph theory. Specifically, there are bipartite graph classes whose biadjacency matrices correspond to matrices in $\oned$ and $\lc$, respectively. For $\oned$, this graph class is that of interval bigraphs, and the equivalence is immediate. For $\lc$, the corresponding graph classes are known under several names and definitions; see the discussion in the next paragraph. In particular, Theorems 1.4 and 3.1 of \cite{saha_permutation_2014}, which relate interval bigraphs to intersections of chain graphs, implicitly imply the following theorem. We provide a direct proof in our terminology.

\begin{restatable}{theorem}{thmOneDSubsetLC}
    It holds that $\oned \subseteq \lc$. 
\end{restatable}
\begin{proof}
  Let $A \in \oned$. Then, $A$ admits a \emph{1D-voter-candidate range} explanation. Therefore, there exists a collection of closed intervals on the real line \(I_a=[l_a,r_a]\) for all voters and candidates such that $A[i][c]=1$ if and only if $I_c \cap I_i \neq \emptyset$. Order the voters and candidates such that $x \succ y$ if  and only if $l_x \leq l_y$, for all voters and candidates, subject to breaking ties arbitrarily. We claim that \( \succ\) witnesses that \(A\) admits a \emph{linearly-consistent} explanation. Consider any voters \(i, j\) and any candidates \(a,b\) such that \(i \succ j\), \(a \succ b\), \(A[i][b]=1\) and \(A[j][a]=1\). We need to prove that \(A[i][a]=1\). Assume for the sake of contradiction that \(A[i][a]=0\). Then, there are two possible cases: 

 \smallskip
\noindent\textbf{Case 1.} \(I_a\) lies entirely to the left of \(I_i\), i.e., \(r_a < l_i\). Combining with \(l_i \le l_j\) gives \(r_a < l_j\), so \(I_a \cap I_j = \emptyset\), contradicting the fact that \(A[j][a]=1\).
 
\smallskip
\noindent\textbf{Case 2.} \(I_i\) lies entirely to the left of \(I_a\), i.e., \(r_i < l_a\). Combining with \(l_a \le l_b\) gives \(r_i < l_b\), so \(I_i \cap I_b = \emptyset\), contradicting the fact that \(A[i][b]=1\).

 In both cases we reached a contradiction, therefore $A[i][a]=1$ and thus $\succ$ witnesses that $A$ admits a \emph{linearly-consistent} explanation.
\end{proof}

For the next statement, implying that $\oned \subsetneq \lc$, the connection is slightly less immediate. \cite{das_interval_1989} show that the class of \emph{interval digraphs} is contained in the class of \emph{digraphs with Ferrers dimension at most 2}. Although these are classes of directed graphs, analogous definitions exist for bipartite graphs, where they correspond to the classes of VCI and LC approval matrices, respectively. However, we could not find the corresponding statements for bipartite graph classes in the literature and therefore translate the proof of \cite{das_interval_1989} to the bipartite case.

\begin{restatable}{theorem}{thmProperSubclass} \label{LC not oned} 
    There exists $A \in\mathcal{A}^{\text{LC}}$ such that $A \notin \oned$.
\end{restatable}
To prove this, we require the following two lemmas. 
\begin{lemma}\label{L-R colouring}
    If $A \in \oned$, then there exists a permutation of columns and rows of $A$ such that every $0$-entry of $A$ can be labeled either $L$ or $R$ in a way such that:
    \begin{itemize}
        \item every entry labeled $L$ has only $0$'s, all labeled $L$, to its right in the same row, and
        \item every entry labeled $R$ has only $0$'s, all labeled $R$, below it in the same column.
    \end{itemize}
\end{lemma}

\begin{proof}
    Since $A \in \oned$ there exists a collection of closed intervals on the real line \(I_a=[l_a,r_a]\) for all voters and candidates such that $A[i][c]=1$ if and only if $I_c \cap I_i \neq \emptyset$. Order the voters and candidates such that $x \succ y$ if and only if $l_x \leq l_y$ for all voters and candidates, subject to breaking ties arbitrarily. Permute the columns and rows of $A$ with respect to $\succ$. Now, consider any $0$-entry $A[i][c]=0$.  Since \(I_i\cap I_c=\emptyset\), the two intervals are disjoint, so exactly one of the following holds:
\begin{itemize}
    \item \(I_i\) lies entirely to the left of \(I_c\), that is, \(r_i < l_c\);
    \item \(I_i\) lies entirely to the right of \(I_c\), that is, \(r_c < l_i\).
\end{itemize}
In the first case, label the entry \(L\). In the second case, label it \(R\).

We now verify that these labels satisfy the required properties.
First, suppose that \(A[i][c]=0\) is labeled \(L\). Then \(I_i\) lies entirely to the left of \(I_c\), so
\(r_i < l_c.\)
Let \(d\) be any candidate whose column lies to the right of column \(c\). By construction of the column order,
\(l_c \le l_d.\)
Hence
\(r_i < l_d,\)
which means that \(I_i\) also lies entirely to the left of \(I_d\). Therefore \(A[i][d]=0\), and it is labeled \(L\). So every \(L\) has only $0$'s, all labeled $L$, to its right.

Second, suppose that \(A[i][c]=0\) is labeled \(R\). Then \(I_i\) lies entirely to the right of \(I_c\), so
\(r_c< l_i.\)
Let \(j\) be any voter whose row lies below row \(i\). By construction of the row order,
\(l_i \le l_j.\)
Hence
\(r_c < l_j,\)
which means that \(I_j\) also lies entirely to the right of \(I_c\). Therefore \(A[j][c]=0\), and it is labeled \(R\). So every \(R\) has only $0$'s, all labeled $R$, below it.
Thus, after suitably permuting the rows and columns, every \(0\)-entry can be labeled by \(L\) or \(R\) so that the claimed properties hold.
\end{proof}

\begin{lemma} \label{2x2 colour}
    Let $A \in \oned$ be a 0-1 matrix. If, for any voters $i,j$ and candidates $a,b$ such that $A[i][a]=1, \quad A[i][b]=0, \quad A[j][a]=0, \quad A[j][b]=1 $,  then the $0$-entries $A[i][b] \text{ and } A[j][a]$ need to be labeled differently.
\end{lemma}
\begin{proof}
Suppose the rows and columns are permuted such
that the top-left of the four entries is either
$A[i][b]=0$ or $A[j][a]=0$. This immediately
contradicts Lemma~\ref{L-R colouring} as then this
entry is both above a~$1$ and to the left of a~$1$.

So, without loss of generality, we may assume that 
the top-left of the four entries is $A[i][a]=1$.
Then the entry $A[i][b]=0$ is top-right and above 
$A[j][b]=1$. Thus, by Lemma~\ref{L-R colouring},
$A[i][b]=0$ cannot be labeled $R$, and thus must be
labeled $L$.
Similarly, the entry $A[j][a]=0$ is bottom-left and to the left of $A[j][b]=1$. Thus, by \Cref{L-R colouring},
$A[j][a]=0$ cannot be labeled $L$, and thus must be
labeled $R$.
Therefore, the two $0$-entries of any such submatrix of $A$ must be labeled differently.
\end{proof}
We can now verify that the VCI domain is a strict subset of the LC domain.
\begin{proof}[Proof of Theorem~\ref{LC not oned}]
Consider the following 0-1 matrix $A$:
    \[
\begin{array}{c|ccccccc}
 & a_1 & a_2 & a_3 & a_4 & a_5 & a_6 & a_7 \\
\hline
v_1 & 1 & 1 & 1 & 0 & 0 & 0 & 0 \\
v_2 & 1 & 1 & 1 & 1 & 1 & 0 & 0 \\
v_3 & 1 & 1 & 1 & 1 & 1 & 1 & 0 \\
v_4 & 0 & 1 & 1 & 1 & 1 & 1 & 1 \\
v_5 & 0 & 1 & 1 & 1 & 1 & 0 & 1 \\
v_6 & 0 & 0 & 1 & 1 & 0 & 0 & 0 \\
v_7 & 0 & 0 & 0 & 1 & 1 & 0 & 1
\end{array}
\]

 One can easily verify that $A \in\mathcal{A}^{\text{LC}}$ for any joint order $\succ$ of voters and candidates that respects both the row ordering of the voters and the column ordering of the candidates, as shown in the representation of $A$ above.
 
By \Cref{2x2 colour} and starting with the \(2\times 2\) submatrix given by $v_1,v_7, a_3, a_4$, we see that the entries $v_1a_4$ and $v_7a_3$ must be labeled differently. Fixing $v_1,v_7,a_4$ and iterating over $a_3,a_2,a_1$ we get that $v_7a_1,v_7a_2$ and $v_4a_3$ must be labeled the same way as $v_1a_4$. Then by replacing in the argument above $v_7$ with $v_6$ and iterating over $a_1,a_2$ we see that $v_6a_1$ and $v_6a_2$ must be labeled the same as $v_1a_4$. Repeating this argument appropriately, we end up with two non-trivial possible configurations of labeling.

\begin{minipage}{.5\textwidth} 
$$\begin{array}{c|ccccccc}
 & a_1 & a_2 & a_3 & a_4 & a_5 & a_6 & a_7 \\
\hline
v_1 & 1 & 1 & 1 & L & L & L & L \\
v_2 & 1 & 1 & 1 & 1 & 1 & L & L \\
v_3 & 1 & 1 & 1 & 1 & 1 & 1 & L \\
v_4 & R & 1 & 1 & 1 & 1 & 1 & 1 \\
v_5 & R & 1 & 1 & 1 & 1 & R & 1 \\
v_6 & R & R & 1 & 1 & L & 0 & L \\
v_7 & R & R & R & 1 & 1 & R & 1
\end{array}$$
\end{minipage}
\begin{minipage}{.5\textwidth} 
$$\begin{array}{c|ccccccc}
 & a_1 & a_2 & a_3 & a_4 & a_5 & a_6 & a_7 \\
\hline
v_1 & 1 & 1 & 1 & R & R & R & R \\
v_2 & 1 & 1 & 1 & 1 & 1 & R & R \\
v_3 & 1 & 1 & 1 & 1 & 1 & 1 & R \\
v_4 & L & 1 & 1 & 1 & 1 & 1 & 1 \\
v_5 & L & 1 & 1 & 1 & 1 & L & 1 \\
v_6 & L & L & 1 & 1 & R & 0 & R \\
v_7 & L & L & L & 1 & 1 & L & 1
\end{array}$$
\end{minipage}
Here we have left the entry $v_6a_6$ unlabeled on purpose. 
Next we describe a forbidden configuration of labels that occurs in both scenarios.
Suppose one can identify rows \(r_1, r_2, r_3, r_4\) and columns \(c_1, c_2, c_3, c_4\) in $A$ that are labeled the following way:
\begin{itemize}
    \item \(L\) labels \(r_1 c_4\), \(r_2 c_4\), and \(r_2 c_3\). The remaining entries of rows \(r_1\) and \(r_2\) are \(1\);
    \item \(R\) labels \(r_4 c_1\), \(r_4 c_2\), and \(r_3 c_2\). The remaining entries of columns \(c_1\) and \(c_2\) are \(1\);
    \item Row \(r_3\) contains at least two \(L\)-entries, and column \(c_3\) contains at least two \(R\)-entries.
\end{itemize}
Then \(A \notin \oned \).
The reason is that, by Lemma~\ref{L-R colouring}, an \(L\) may only be followed by \(L\)'s in its row and an \(R\) may only be followed by \(R\)'s in its column.
 
Then,  entry \(r_1 c_4\) is labeled \(L\) and the rest of \(r_1\) contains only $1$-entries, so \(c_4\) must be the rightmost column in the entire matrix. Similarly, \(r_2 c_4\) and \(r_2 c_3\) are labeled \(L\) and the rest of \(r_2\) contains only $1$-entries, which forces \(c_3\) to be the second rightmost column in the entire matrix. By the symmetric argument using \(r_4 c_1\), \(r_4 c_2\), and \(r_3 c_2\), the rows \(r_4\) and \(r_3\) are the bottom and second-to-bottom rows in the entire matrix, respectively. 
 
Now consider the entry \(r_3 c_3\). We know that $r_3$ contains at least 2 $L$'s. This means that if $r_3c_3$ is labeled $R$ then, since $c_3$ is the second rightmost column there must be at least one column $c_j$ on the left of $c_3$ such that $r_3c_j$ is labeled $L$, contradicting \Cref{L-R colouring}. Therefore $r_3c_3$ must be labeled $L$. In that case because $c_3$ contains at least 2 $R$'s and $r_3$ is the second to last row, there must exist a row $r_i$ above $r_3$ such that $r_ic_3$ is labeled $R$; again, this contradicts \Cref{L-R colouring}. In words, $r_3c_3$ must simultaneously be labeled both $R$ and $L$, a contradiction.

Notice that in both possible configurations one can find rows $r_1,r_2,r_3,r_4$ and columns $c^\prime_1,c^\prime_2,c^\prime_3,c^\prime_4$ satisfying the assumptions above and therefore $A \notin \oned$. In the first case, choose $(r_1,r_2,r_3,r_4)=(v_3,v_2,v_6,v_7)$ and $(c^\prime_1,c^\prime_2,c^\prime_3,c^\prime_4)=(c_3,c_2,c_6,c_7)$. In the second case, 
$(r_1,r_2,r_3,r_4)=(v_4,v_5,v_6,v_1)$ and $(c^\prime_1,c^\prime_2,c^\prime_3,c^\prime_4)=(c_4,c_5,c_6,c_1)$.
\end{proof}

\paragraph{An Alternative Interpretation of LC Profiles.}

As mentioned above, the class of bipartite graphs whose biadjacency matrix corresponds to $\lc$ matrices has been studied in the graph theory literature under a number of different names and definitions: \emph{interval-containment} bigraphs; bipartite graphs whose bipartite complement is a \emph{circular-arc graph}; intersections of two \emph{chain graphs} with the same bipartition \cite{saha_permutation_2014}; two-directional \emph{orthogonal ray} bipartite graphs; \emph{chordal bipartite} graphs with no edge-asteroids \cite{SHRESTHA20101650}; and \emph{signed-interval} bipartite graphs~\cite{hell_min-orderable_2020}. All of these definitions have been shown to be equivalent in the cited works, which, in itself can be seen as a further motivation for the LC domain. 

Among all of these graph classes, interval-containment bigraphs have a particularly natural interpretation from a social choice perspective and nicely connect the LC domain to the VCI domain. We define the corresponding class of approval matrices: the class of \emph{voter-contains-candidate-interval}~(VCCI) matrices $\mathcal{A}^\text{VCCI}$ contains all matrices for which there exists a \emph{1D-voter-contains-candidate range} explanation. That is, we can assign each voter and each candidate an interval on the real line $\mathbb{R}$, such that, for every voter $i$ and candidate $c$, it holds that $A[i][c]=1$ if and only if the interval corresponding to $c$ is contained within the interval corresponding to $i$, that is $I_c\subseteq I_i$.

While the equivalence result between LC and VCCI can be obtained by chaining several results by \cite{saha_permutation_2014} and \cite{hell_min-orderable_2020}, we provide an independent and direct proof of this equivalence.

\begin{restatable}{theorem}{thmLCVCCI}\label{thm:equi-LC-VCCI}
It holds that $\lc = \mathcal{A}^\text{VCCI}$.
\end{restatable}
\begin{proof}
First, we prove that if $A \in \mathcal{A}^\text{LC}$ then $A \in \mathcal{A}^\text{VCCI}$.
Let $\succ$ be a linear order verifying $A \in \mathcal{A}^\text{LC}$. Let $\succ_N$ be the restriction of $\succ$ onto the voters, and let $\succ_C$ be the restriction of $\succ$ onto the candidates. By definition, for all voters $i,j$ and candidates $c,d$, we have
$$i \succ_N j, c \succ_C d, A[i][d]=1, A[j][c]=1 \quad \Rightarrow \quad A[i][c]=1.$$
    Let \(\rho(i)\in\{1,\dots,n\}\) be the position of voter \(i\) in the order \(\succ_N\), and let \(\kappa(c)\in\{1,\dots,m\}\) be the position of candidate \(c\) in the order \(\succ_C\), where position \(1\) is first in the order. Thus,
by definition, $i \succ_N j$ if and only if $\rho(i)<\rho(j)$ and $c \succ_C d$ if and only if $\kappa(c)<\kappa(d)$.

For each voter \(i\in N\), define
\(
r(i):=\max\{\kappa(c): A[i][c]=1\}
\)
and let \(R(i)\) be the unique candidate with \(\kappa(R(i))=r(i)\), i.e., the last candidate approved by \(i\) in the order \(\succ_C\).
Similarly, for each candidate \(c\in C\), define
\(
u(c):=
\max\{\rho(i): A[i][c]=1\}
\)
and let \(U(c)\) be the unique voter with \(\rho(U(c))=u(c)\), i.e., the last voter approving \(c\) in the order \(\succ_N\).

We claim that, for every voter \(i\) and candidate \(c\),
$A[i][c]=1$ if and only if $\rho(i)\le u(c)$
and $\kappa(c)\le r(i)$.
The forward implication is immediate. 
For the converse, suppose that \(\rho(i)\le u(c)\) and \(\kappa(c)\le r(i)\). If \(i=U(c)\) or \(c=R(i)\) then \(A[i][c]=1\) immediately. Thus we may assume \(i\neq U(c)\) and \(c\neq R(i)\). So, since \(\rho(i)< u(c)=\rho(U(c))\), we have \(i\succ_N U(c)\). Similarly, since \(\kappa(c) < r(i)=\kappa(R(i))\), we have \(c\succ_C R(i)\). Applying the linearly-consistent condition to \(i\), \(U(c)\), \(c\), and \(R(i)\), we conclude that \(A[i][c]=1\), as desired.

Next, assign to each voter \(i\in N\) the interval
\(
I_i=[\rho(i),\, n+r(i)]
\)
and to each candidate \(c\in C\) the interval
\(
J_c=[u(c),\, n+\kappa(c)].
\)
Then $J_c\subseteq I_i $ if and only if $\rho(i)\le u(c)$ and $n+\kappa(c)\le n+r(i)$; that is,
$J_c\subseteq I_i$ if and only if $\rho(i)\le u(c)$ and $\kappa(c)\le r(i)$.
By the claim above, this is equivalent to \(A[i][c]=1\). Therefore, $A[i][c]=1$ if and only if $J_c\subseteq I_i$.

\bigskip

We move to the other direction of the proof, i.e., we show that if $A \in \mathcal{A}^\text{VCCI}$ then  $A \in \lc$.
Assuming $A \in \mathcal{A}^\text{VCCI}$, there exist
in $\mathbb{R}$ intervals $I_i=[L_i,R_i]$, for each voter $i$,
and $J_c=[\ell_c,r_c]$, for each candidate $c$,
such that $A[i][c]$ if and only if $J_c\subseteq I_i$.

Arrange the voters by increasing left-endpoints, \(L_i\), to give a linear order \(\succ_N\) (breaking ties arbitrarily). 
Similarly, arrange the candidates by increasing right-endpoints, \(r_c\), to give a linear order \(\succ_C\)  (breaking ties arbitrarily). Thus,
$i\succ_N j \Longrightarrow L_i\le L_j$ and
$c\succ_C d \Longrightarrow r_c\le r_d$.

Take voters \(i,j\in N\) and candidates \(c,d\in C\) such that \(i\succ_N j\), \(c\succ_C d\), \(A[i][d]=1\), and \(A[j][c]=1\). We want to show that \(A[i][c]=1\). Since \(A[j][c]=1\), we know that \(J_c\subseteq I_j\) and thus \(L_j\le \ell_c\). Since \(i\succ_N j\), we also have \(L_i\le L_j\). Therefore \(L_i\le \ell_c\).

Similarly, because \(A[i][d]=1\), we know that \(J_d\subseteq I_i\) and thus \(r_d\le R_i\). Since \(c\succ_C d\), we also have \(r_c\le r_d\). Therefore \(r_c\le R_i\). Consequently, \(L_i\le \ell_c\) and \(r_c\le R_i\), which implies \(J_c\subseteq I_i\). Thus \(A[i][c]=1\), as desired.

Hence any order over the candidates and voters \(\succ\) respecting the orders \(\succ_N\) and \(\succ_C\) verifies that $A\in \lc$.
\end{proof}

This concludes our discussion of the LC domain.

\section{Committee Winner Determination is Hard in the Tree Representation Domain}\label{sec:tree}

The class of \emph{tree representation (TR)} matrices $\tree$ contains all matrices for which there is an explanation via a tree $T$, i.e., we can assign each voter and each candidate a subtree of $T$, so that, for every voter $i$ and candidate $j$ we have $A[i][j] = 1$  if and only if the two corresponding subtrees intersect. 
Observe that $\oned$ is the special case of $\tree$
where the tree $T$ is a path $P$.
However, the winner determination problem is hard
for the class $\tree$.

\begin{theorem}\label{thm:tree}
The winner determination problem for Thiele rules
is NP-complete for the tree representation domain.
This holds even for Chamberlin-Courant weight functions  in the special cases where either (i) the candidate subtrees are singletons, or (ii) the voter subtrees are singletons.
\end{theorem}

\begin{proof}
It is straightforward to verify that the winner determination problem is in NP. To prove hardness we use a reduction from the set cover problem. Given a universe $U=[n]$ of elements, a collection
$\mathcal{S}=\{S_1, S_2,\dots , S_m\}$ of subsets of elements, and an integer $k$. The {\sc set cover} problem asks: are there $k$ subsets whose union is the universe $U$? 

(i) Consider first the case where the candidates have subtrees that are singleton vertices. 
Let $T$ be a star with centre $c$ and $m$ adjacent leaves 
$\{v_1, v_2,\dots , v_m\}$. 
We create a candidate $j$ for each subset $S_j\in \mathcal{S}$
and a voter $i$ for each element $i\in U$.
Specifically, let there be $m$ candidates, with candidate $j$ located at leaf $v_j$, that is, it has a singleton subtree. Let there be $n$ voters where voter $i$ has a subtree (substar) formed by the vertices $\{v_j: i\in S_j\}\cup\{c\}$. Thus voter $i$ has approval set $C_i=\{j: i\in S_j\}$.
Finally, let the target score be $\tau=n$.

With the Chamberlin-Courant weight function, voter $i$
is happy (has utility $1$) if and only if at least one vertex of $C_i$ is selected in the committee
(observe that the centre $c$ does not correspond to a candidate so cannot be selected).
It immediately follows that there is a set cover of
cardinality $k$ if and only if there is a committee of cardinality $k$ with total score $\tau=n$.

(ii) Next consider the case where the voters have subtrees that are singleton vertices. 
Now let $T$ be a star with centre $c$ and $n$ adjacent leaves 
$\{v_1, v_2,\dots , v_n\}$. 
We create a candidate $j$ for each subset $S_j\in \mathcal{S}$
and a voter $i$ for each element $i\in U$.
Specifically, let there be $n$ voters, with voter $i$ located at leaf $v_i$, that is, it has a singleton subtree. 
Let there be $m$ candidates where candidate $j$ has a subtree (substar) formed by the vertices $\{v_i: i\in S_j\}\cup\{c\}$. Thus candidate $j$ has supporter set $N_j=\{i: i\in S_j\}$. Finally, let the target score be $\tau=n$.

With the Chamberlin-Courant weight function, voter $i$
is happy (has utility $1$) if and only if at least one candidate of its approval set $C_i=\{j: i\in S_j\}$ is selected in the committee. Again, it immediately follows that there is a set cover of cardinality $k$ if and only if there is a committee of cardinality $k$ with total score $\tau=n$.
\end{proof}

Consider Case (i) where every candidate subtree is a singleton vertex in the tree. A more specialized case where, in addition, every vertex of the tree is a candidate has been studied before in the literature; see, for example, \cite{Yang19,ELKIND2023114039}.
We remark that this special case remains hard.

\begin{corollary}\label{cor:tree}
The winner determination problem for Thiele rules
is NP-complete in tree representation domains
for Chamberlin-Courant weight functions 
in the special cases where the candidate subtrees are singletons and every vertex of the tree is a candidate.
\end{corollary}
\begin{proof}
We modify the above reduction from set cover.
Let $T$ be a star with centre $c$ and $m$ adjacent leaves $\{v_1, v_2,\dots, v_m\}$. 
Let there be $m+1$ candidates, one for each vertex of $T$. Specifically, there is a candidate at the centre $c$ and a candidate $j$, for each subset $S_j\in \mathcal{S}$, located at leaf $v_j$.

Let there be $n+m\cdot L$ voters where $L\ge n$ is a large number.
First, there is voter $i$, for each element $i\in U$,
with a subtree (substar) formed by the vertices $C_i=\{v_j: i\in S_j\}\cup\{c\}$. 
Second, for each leaf $v_j$, there are $L$ dummy voters 
who only approve of candidate $j$. 
Finally, let the target score be $\tau=n +k\cdot L$.

With the Chamberlin-Courant weight function, voter $i$
is happy (has utility $1$) if and only if at least one vertex of $C_i\setminus \{c\}$ is selected in the committee. A dummy voter at leaf $v_j$ is happy if and only if candidate $j$ is selected.

Observe that if $k$ leaf candidates are selected in the committee then the total score is $k\cdot L + n$
if and only if it corresponds to a set cover.
If $k-1$ leaf candidates are selected along with the
centre candidate $c$ then the maximum possible score is
$(k-1)\cdot L + n \le k\cdot L$. 
\end{proof}

\section{Discussion}
We showed that every restricted domain satisfying the key lemma (\Cref{lem:chris}) admits a polynomial-time algorithm for all generalized Thiele rules, resolving a long-standing open question. For a natural subclass of Thiele rules, including PAV, we further prove that every optimal extreme point of the standard LP formulation is integral; thus, computing a basic optimal LP solution suffices. We also identify a previously unknown connection between the VCI and LC domains, give a more natural interpretation of the LC domain, and show that polynomial tractability breaks down on the tree representation domain.
For future work, we see several directions. First, it would be interesting to characterize the restricted domains on which Thiele rules are tractable in polynomial time. Second, \Cref{lem:chris} may be useful for proving tractability of other voting rules that are NP-hard in general, such as max-Phragmén \citep{brill2024phragmen}. Finally, further structural questions for restricted domains in approval-based committee voting deserve attention, for instance the compatibility of strong proportionality properties with monotonicity properties, which remains open in general; see \citep{lackner23abc_book}.

\bibliographystyle{abbrvnat}
\bibliography{ref}

@inproceedings{DBLP:conf/aaai/Peters18,
  author       = {Dominik Peters},
  title        = {Single-Peakedness and Total Unimodularity: New Polynomial-Time Algorithms
                  for Multi-Winner Elections},
  booktitle    = {Proceedings of the 32nd {AAAI} Conference on Artificial Intelligence ({AAAI})},
  pages        = {1169--1176},
  year         = {2018}
}

@book{DBLP:books/daglib/0090562,
  author       = {Alexander Schrijver},
  title        = {Theory of Linear and Integer Programming},
  NOOPseries       = {Wiley-Interscience series in discrete mathematics and optimization},
  publisher    = {Wiley},
  year         = {1986}
}

@inproceedings{AzizGGMMW15,
  author       = {Haris Aziz and
                  Serge Gaspers and
                  Joachim Gudmundsson and
                  Simon Mackenzie and
                  Nicholas Mattei and
                  Toby Walsh},
  title        = {Computational Aspects of Multi-Winner Approval Voting},
  booktitle    = {Proceedings of the 2015 International Conference on Autonomous Agents and Multiagent Systems ({AAMAS})},
  pages        = {107--115},
  year         = {2015}
}

@inproceedings{BredereckF0KN20_aaai,
  author       = {Robert Bredereck and
                  Piotr Faliszewski and
                  Andrzej Kaczmarczyk and
                  Dusan Knop and
                  Rolf Niedermeier},
  title        = {Parameterized Algorithms for Finding a Collective Set of Items},
  booktitle    = {Proceedings of the Thirty-Fourth {AAAI} Conference on Artificial Intelligence ({AAAI})},
  pages        = {1838--1845},
  year         = {2020},
  volume={34},
  number={2}
}

@inproceedings{ElkindL15_ijcai15,
  author       = {Edith Elkind and
                  Martin Lackner},
  title        = {Structure in Dichotomous Preferences},
  booktitle    = {Proceedings of the 24th International Joint Conference on Artificial Intelligence ({IJCAI})},
  pages        = {2019--2025},
  year         = {2015}
}

@article{FaliszewskiSST18_scw,
  author       = {Piotr Faliszewski and
                  Piotr Skowron and
                  Arkadii Slinko and
                  Nimrod Talmon},
  title        = {Multiwinner Analogues of the Plurality Rule: {A}xiomatic and Algorithmic Perspectives},
  journal      = {Social Choice and Welfare},
  volume       = {51},
  number       = {3},
  pages        = {513--550},
  year         = {2018}
}

@inproceedings{GodziszewskiB0F21_vci_aaai21,
  author       = {Micha\l{} Godziszewski and
                  Pawe\l{} Batko and
                  Piotr Skowron and
                  Piotr Faliszewski},
  title        = {An Analysis of Approval-Based Committee Rules for {2D}-{E}uclidean Elections},
  booktitle    = {Proceedings of the 35th {AAAI} Conference on Artificial Intelligence ({AAAI})},
  pages        = {5448--5455},
  year         = {2021},
}

@book{GLS93,
  author       = {Grötschel, M. and Lovász, L. and Schrijver, A.},
  title        = {Geometric Algorithms and Combinatorial Optimization},
  publisher    = {Springer},
  edition        = {Second},
  year         = {1988}
}

@article{Kha79,
author       = {Khachiyan, L.},
  title        = {A Polynomial Algorithm in Linear Programming},
 journal    = {Doklady Akademii Nauk SSSR},
  pages        = {1093-1096},
  year         = {1979},
  volume= {244}
}

@article{SkowronFL16_aij_set_of_items,
  author       = {Piotr Skowron and
                  Piotr Faliszewski and
                  J{\'{e}}r{\^{o}}me Lang},
  title        = {Finding a Collective Set of Items: {F}rom Proportional Multirepresentation to Group Recommendation},
  journal      = {Artificial Intelligence},
  volume       = {241},
  pages        = {191--216},
  year         = {2016}
}

@inproceedings{SornatWX22_ijcai,
  author       = {Krzysztof Sornat and
                  Virginia {Vassilevska Williams} and
                  Yinzhan Xu},
  title        = {Near-Tight Algorithms for the {C}hamberlin-{C}ourant and {T}hiele Voting Rules},
  booktitle    = {Proceedings of the 31st International Joint Conference on Artificial Intelligence ({IJCAI})},
  pages        = {482--488},
  year         = {2022}
}

@incollection{Thiele95,
  author    = {Thorvald Thiele},
  booktitle = {Oversigt over det Kongelige Danske Videnskabernes Selskabs Forhandlinger (in {D}anish)},
  pages     = {415--441},
  title     = {Om Flerfoldsvalg},
  year      = {1895},
  publisher = {K{\o}benhavn: A.F. H{\o}st}
}

@inproceedings{Yang19,
  author       = {Yongjie Yang},
  title        = {On the tree representations of dichotomous preferences},
  booktitle    = {Proceedings of the 28th International Joint Conference on Artificial Intelligence ({IJCAI})},
  pages        = {644--650},
  year         = {2018}
}

@article{procaccia2008complexity,
  title={On the complexity of achieving proportional representation},
  author={Procaccia, Ariel and Rosenschein, Jeffrey and Zohar, Aviv},
  journal={Social Choice and Welfare},
  volume={30},
  number={3},
  pages={353--362},
  year={2008},
  publisher={Springer}
}

@inproceedings{dong25interlacing,
author = {Dong, Chris and Bullinger, Martin and W\k{a}s, Tomasz and Birnbaum, Larry and Elkind, Edith},
title = {Selecting Interlacing Committees},
year = {2025},
publisher = {International Foundation for Autonomous Agents and Multiagent Systems},
booktitle = {Proceedings of the 24th International Conference on Autonomous Agents and Multiagent Systems (AAMAS)},
pages = {630–638},
}

@InProceedings{10.1007/978-3-031-22832-2_18,
author="Pierczy{\'{n}}ski, Grzegorz
and Skowron, Piotr",
title="Core-Stable Committees Under Restricted Domains",
booktitle="Proceedings of the 18th International Conference on Web and Internet Economics (WINE)",
year="2022",
NOOPpublisher="Springer International Publishing",
NOOPaddress="Cham",
pages="311--329",
}

@article{das_interval_1989,
	title = {Interval digraphs: {An} analogue of interval graphs},
	volume = {13},
	number = {2},
	journal = {Journal of Graph Theory},
	author = {Das, S. and Sen, M. and Roy, A. and West, D.},
	year = {1989},
	pages = {189--202},
}

@article{saha_permutation_2014,
	address = {NLD},
	title = {Permutation bigraphs and interval containments},
	volume = {175},
	number = {C},
	journal = {Discrete Applied Mathematics},
	publisher = {Elsevier Science Publishers B. V.},
	author = {Saha, Pranab and Basu, Asim and Sen, Malay  and West, Douglas},
	month = oct,
	year = {2014},
	pages = {71--78},
}

@article{ELKIND2023114039,
title = {Justifying groups in multiwinner approval voting},
journal = {Theoretical Computer Science},
volume = {969},
pages = {114039},
year = {2023},
author = {Edith Elkind and Piotr Faliszewski and Ayumi Igarashi and Pasin Manurangsi and Ulrike Schmidt-Kraepelin and Warut Suksompong},
publisher="Springer International Publishing"
}

@book{lackner23abc_book,
  author       = {Martin Lackner and
                  Piotr Skowron},
  title        = {Multi-Winner Voting with Approval Preferences},
  publisher    = {Springer},
  year         = {2023}
}

@inproceedings{boehmer2024approval,
  title={Approval-based committee voting in practice: A case study of (over-) representation in the Polkadot blockchain},
  author={Boehmer, Niclas and Brill, Markus and Cevallos, Alfonso and Gehrlein, Jonas and S{\'a}nchez-Fern{\'a}ndez, Luis and Schmidt-Kraepelin, Ulrike},
  booktitle={Proceedings of the AAAI Conference on Artificial Intelligence (AAAI)},
  volume={38},
  number={9},
  pages={9519--9527},
  year={2024}
}

@inproceedings{cevallos2021verifiably,
  title={A verifiably secure and proportional committee election rule},
  author={Cevallos, Alfonso and Stewart, Alistair},
  booktitle={Proceedings of the 3rd ACM Conference on Advances in Financial Technologies},
  pages={29--42},
  year={2021},
}

@incollection{aziz2020participatory,
  title={Participatory budgeting: Models and approaches},
  author={Aziz, Haris and Shah, Nisarg},
  booktitle={Pathways Between Social Science and Computational Social Science: Theories, Methods, and Interpretations},
  editor={Rudas, T. and P\'eli, G.},
  pages={215--236},
  year={2020},
  publisher={Springer},
}

@article{SHRESTHA20101650,
title = {On orthogonal ray graphs},
journal = {Discrete Applied Mathematics},
volume = {158},
number = {15},
pages = {1650-1659},
year = {2010},
author = {Anish Shrestha and Satoshi Tayu and Shuichi Ueno},
}

@article{hell_min-orderable_2020,
	title = {Min-{Orderable} {Digraphs}},
	volume = {34},
	number = {3},
	journal = {SIAM Journal on Discrete Mathematics},
	author = {Hell, Pavol and Huang, Jing and McConnell, Ross and Rafiey, Arash},
	year = {2020},
	pages = {1710--1724},
}

@article{SPINRAD1987279,
title = {Bipartite permutation graphs},
journal = {Discrete Applied Mathematics},
volume = {18},
number = {3},
pages = {279-292},
year = {1987},
author = {Jeremy Spinrad and Andreas Brandstädt and Lorna Stewart},
}

@article{aziz2017justified,
  title={Justified representation in approval-based committee voting},
  author={Aziz, Haris and Brill, Markus and Conitzer, Vincent and Elkind, Edith and Freeman, Rupert and Walsh, Toby},
  journal={Social Choice and Welfare},
  volume={48},
  number={2},
  pages={461--485},
  year={2017},
  publisher={Springer},
}

@inproceedings{aziz2018complexity,
  title={On the complexity of extended and proportional justified representation},
  author={Aziz, Haris and Elkind, Edith and Huang, Shenwei and Lackner, Martin and S{\'a}nchez-Fern{\'a}ndez, Luis and Skowron, Piotr},
  booktitle={Proceedings of the AAAI Conference on Artificial Intelligence (AAAI)},
  volume={32},
  pages= {902--909},
  number={1},
  year={2018},
}

@inproceedings{peters2020proportionality,
  title={Proportionality and the limits of welfarism},
  author={Peters, Dominik and Skowron, Piotr},
  booktitle={Proceedings of the 21st ACM Conference on Economics and Computation (EC)},
  pages={793--794},
  year={2020}
}

@article{brill2024approval,
  title={Approval-based apportionment},
  author={Brill, Markus and G{\"o}lz, Paul and Peters, Dominik and Schmidt-Kraepelin, Ulrike and Wilker, Kai},
  journal={Mathematical Programming},
  volume={203},
  number={1},
  pages={77--105},
  year={2024},
  publisher={Springer}
}

@inproceedings{dudycz2021tight,
  title={Tight approximation for proportional approval voting},
  author={Dudycz, Szymon and Manurangsi, Pasin and Marcinkowski, Jan and Sornat, Krzysztof},
  booktitle={Proceedings of the Twenty-Ninth International Conference on International Joint Conferences on Artificial Intelligence (IJCAI)},
  pages={276--282},
  year={2020}
}

@inproceedings{lassota2026algorithms,
  title={Algorithms for Structured Elections Under Thiele Voting Rules},
  author={Lassota, Alexandra and Sornat, Krzysztof},
  booktitle={Proceedings of the AAAI Conference on Artificial Intelligence (AAAI)},
  volume={40},
  number={20},
  pages={17084--17092},
  year={2026}
}

@article{manurangsi2026polynomial,
  title={Polynomial-Time Algorithm for Thiele Voting Rules with Voter Interval Preferences},
  author={Manurangsi, Pasin and Sornat, Krzysztof},
  journal={arXiv preprint arXiv:2604.05953},
  year={2026}
}

@article{brill2024phragmen,
  title={Phragm{\'e}n’s voting methods and justified representation: M. Brill et al.},
  author={Brill, Markus and Freeman, Rupert and Janson, Svante and Lackner, Martin},
  journal={Mathematical programming},
  volume={203},
  number={1},
  pages={47--76},
  year={2024},
  publisher={Springer}
}

\end{document}